\documentclass[final]{alt2026} 
\usepackage[english,vietnamese]{babel}
\usepackage[utf8]{inputenc}
\usepackage[T5]{fontenc}

\title[Privately Learning Decision lists and Differentially Private Winnow]{Privately Learning Decision Lists and a Differentially Private Winnow}

\usepackage{times}

\usepackage{hyperref}       
\usepackage{url}            
\usepackage{booktabs}       
\usepackage{amsfonts}       
\usepackage{nicefrac}       
\usepackage{microtype}      
\usepackage{xcolor}         

\usepackage{graphicx}
\usepackage{amsmath}

\usepackage{array}
\usepackage{bbm}
\usepackage{enumerate}
\usepackage[shortlabels]{enumitem}

\DeclareMathOperator{\sgn}{sgn}
\usepackage{adjustbox} 
\usepackage{verbatim}

\altauthor{%
 \Name{Mark Bun} \Email{mbun@bu.edu}\\
 \addr Boston University
 \AND
 \Name{William Fang} \Email{wfang@bu.edu}\\
 \addr Boston University
}

\begin{document}
\selectlanguage{english}

\maketitle

\begin{abstract}
We give new differentially private algorithms for the classic problems of learning decision lists and large-margin halfspaces in the PAC and online models. In the PAC model, we give a computationally efficient algorithm for learning decision lists with minimal sample overhead over the best non-private algorithms. In the online model, we give a private analog of the influential Winnow algorithm for learning halfspaces with mistake bound polylogarithmic in the dimension and inverse polynomial in the margin. As an application, we describe how to privately learn decision lists in the online model, qualitatively matching state-of-the art non-private guarantees.

\end{abstract}

\section{Introduction}

Differential privacy~\citep{DworkMNS06} is a standard mathematical guarantee of individual-level privacy for algorithmic data analysis tasks, including those that arise in machine learning. In this paper, we give new differentially private learning algorithms, in the canonical PAC and online
learning models, for some of the most fundamental concept classes in computational learning theory: decision lists and large-margin halfspaces


\subsection{Privately PAC Learning Decision Lists}

Our first main contribution concerns privately learning decision lists in the PAC model. The PAC model is the standard statistical learning model for binary classification, in which a learner is presented random examples $(x_1, y_1), \dots, (x_n, y_n)$ where the $x_i$-s are drawn i.i.d. from an arbitrary unknown distribution $\mathcal{D}$ over $\mathcal{X}$, and each $y_i = f(x_i)$ for some unknown target function $f \in \mathcal{C}$. The goal of the learner is to produce a hypothesis $h: \mathcal{X} \to \{-1, 1\}$ that well-approximates $f$ with respect to the distribution $\mathcal{D}$.

A \emph{decision list} is a sequence of ``if-then-else'' rules of the form ``if $\ell_1$ then $b_1$, else if $\ell_2$ then $b_2$, \dots, else $b_r$'' where each $\ell_i$ is a literal and each $b_i$ is a decision bit. 
In his influential paper introducing the model of decision lists, \citet{Riv87b} described a computationally efficient PAC learner for this class. Rivest showed that (arbitrary length) decision lists over $d$ variables are PAC learnable using $\tilde{O}(d)$ samples in $\operatorname{poly}(d)$ time. More generally, if instead of constraining each $\ell_i$ to be a literal, one allows it to be a member of an arbitrary class $\mathcal{F}$ of ``features'', such generalized decision lists are learnable using $\tilde{O}(|\mathcal{F}|)$ samples in $\operatorname{poly}(|\mathcal{F}|)$ time. Rivest's algorithm is a simple and intuitive empirical risk minimization procedure for finding a decision list consistent with a given sample. It proceeds by building a decision list one term at a time by just choosing any candidate term that is consistent with the portion of the sample that has not yet been correctly classified.


In practice, decision lists are widely deployed in high-stakes domains such as healthcare (where they are used to provide risk assessments) and finance. Even in the era of modern machine learning, they hold appeal for simultaneously providing expressivity, efficient learnability, and human interpretability~\citep{RudinCCHSZ22}. However, their compatibility with privacy constraints remains unexplored -- a critical gap for sensitive applications. Furthermore, decision lists represent a computational frontier: non-privately, they are among the most expressive natural classes known to be efficiently PAC-learnable, yet seemingly slightly more general classes (e.g., poly-size DNFs) have resisted the design of efficient algorithms for decades. This raises a fundamental question: Does adding privacy constraints undermine the efficient learnability of such frontier classes?


We address this question for decision lists by giving a differentially private analog of Rivest's algorithm achieving essentially the same guarantees.

\begin{theorem}[Informal Statement of Theorem~\ref{theorem: zCDP DP-GreedyCover}]
Let $\mathcal{F}$ be a set of ``features'' and let $\mathcal{F}$-decision lists denote the class of decision lists where each conditional rule is a member of $\mathcal{F}$. Then there is an $(\varepsilon, \delta)$-differentially private learner for $\mathcal{F}$-decision lists using $\tilde{O}(|\mathcal{F}| \cdot \log(1/\delta)/\varepsilon)$ samples and running in $\operatorname{poly}(|\mathcal{F}|)$ time.
\end{theorem}

\paragraph{Techniques.} A baseline strategy for privately learning any finite concept class is to use the Exponential Mechanism~\citep{McSherryT07} to perform approximate empirical risk minimization~\citep{KasiviswanathanLNRS08}. This works to learn $\mathcal{F}$-decision lists using $\tilde{O}(|\mathcal{F}|)$ samples, but the runtime of a straightforward implementation is exponential in this quantity.

To design a computationally efficient algorithm, we follow Rivest's iterative algorithm, but instead of choosing each new term deterministically, we use the Exponential Mechanism to select a term that has low classification error with respect to the portion of the sample not yet correctly classified. This gives an efficient algorithm, but a na\"{\i}ve privacy analysis based on standard composition theorems yield a learner that seems to require $\tilde{O}(|\mathcal{F}|^{3/2})$ samples. Fortunately, we show that the structure of our iterative Exponential Mechanism-based learner follows the template of earlier approximate Set Cover algorithms~\citep{gupta2009differentially, DBLP:journals/corr/abs-1902-05017} for which a miraculously sharper privacy analysis is known.

\subsection{Private Online Learning of Halfspaces via Winnow}

For our second main contribution, we turn our attention to the more challenging \emph{online mistake-bound} model of learning~\citep{Littlestone88}. This model captures learning from examples under minimal assumptions about how those examples are generated. When applied to binary classification, learning is modeled as a sequential $T$-round game between a learner and an adversary, where in each round $t$, the adversary presents an example $x_t \in \mathcal{X}$, the learner responds with a hypothesis $h_t: \mathcal{X} \to \{-1, 1\}$ and then the adversary reveals the true label $y_t \in \{-1, 1\}$. The learner's goal is produce hypotheses that are competitive in accuracy (as measured by the number of mistakes $h_t(x_t) \ne y_t$) with the best fixed function from a known class $\mathcal{C}$ of Boolean functions $f : \mathcal{X} \to \{-1, 1\}$. In the \emph{realizable} case of online learning, true labels are constrained to be completely consistent with some $f \in \mathcal{C}$ and an ideal learner may aim to predict $f$ with perfect accuracy after making a bounded number of mistakes.

A differentially private online learner is one whose entire sequence of hypotheses $h_t, \dots, h_T$ is differentially private with respect to changing any labeled example $(x_t, y_t)$. Most work on private online learning has focused on the problems of prediction from expert advice and online convex optimization more generally~\citep{DworkNPR10, JainKT12, SmithT13, JainT14, AgarwalS17, Kairouz21b, AsiFKT23, asi2023nearoptimal}. Recently, ~\citet{GolowichLivni2021} honed in on the setting of binary classification described above, giving sublinear (in $T$) mistake bounds for learning any class of finite Littlestone dimension. Here, the Littlestone dimension is a combinatorial parameter that precisely characterizes non-private mistake bound learnability. Golowich and Livni's algorithm is a (technically sophisticated) analog of Littlestone's ``Standard Optimal Algorithm'' for non-private online learning~\citep{Littlestone88}, enhanced with techniques previously used to establish the private PAC learnability of these classes.

With the goal of designing a private online learner for decision lists, we actually tackle the more general problem of learning large-margin halfspaces (a.k.a. linear threshold functions). A halfspace with weight vector $v \in \mathbb{R}^d$, where $\|v\|_1 = 1$, is the function $\sgn(\langle v, x \rangle)$, and we say it has margin $\rho$ if $|\langle v, x\rangle| \ge \rho$ for every $x \in \{-1, 1\}^d$. The other seminal algorithm from Littlestone's original paper, called Winnow, (non-privately) learns a halfspace by maintaining an estimate of the unknown weight vector $v$, updating the weights multiplicatively whenever it makes a prediction error. Littlestone showed that regardless of $T$, the Winnow algorithm is guaranteed to make at most $O(\log d / \rho^2)$ mistakes when learning any $\rho$-margin halfspace. Thus, the algorithm especially shines at attribute-efficient learning, where the ambient dimension $d$ is much larger than the number of variables actually relevant to classification. For example, when learning sparse disjunctions of $k$ out of $d$ variables, for $k \ll d$, the Winnow algorithm described above has mistake bound $O(k^2 \log d)$ with a variant tailored to this class achieving an improved mistake bound of $O(k \log d)$. The Winnow algorithm learns length-$r$ decision lists with mistake bound $2^{O(r)} \log n$, with an improvement to $O(r^{2D} \log n)$ when it is further promised that the decision bits alternate between $+1$ and $-1$ at most $D$ times.

Our second main contribution gives the following differentially private analog of the Winnow algorithm.

\begin{theorem}[Informal Statement of Theorem~\ref{theorem: DP-Winnow privacy and regret}]
Let $\mathcal{C}$ be the class of halfspaces over $\{-1, 1\}$ with margin $\rho$. There is an $(\varepsilon, \delta)$-differentially private online learner for $\mathcal{C}$ incurring regret \\ $O(\operatorname{polylog}(d, T, 1/\delta)/ \rho^6 \varepsilon^4)$ against an oblivious adversary.
\end{theorem}

Our differentially private Winnow has the same qualitative behavior (polylogarithmic in $d$ and inverse polynomial in $\rho$) as its non-private counterpart, and thus similar behavior as Winnow for sparse disjunctions and short decision lists. Note that an explicit dependence of $\log T$ on the time horizon is generally necessary for differentially private online learning, even for the simplest of concept classes~\citep{Cohen24, Dmitriev24}. Section~\ref{section:lower} discusses how this lower bound applies to our setting.

\paragraph{Techniques.} In a bit more detail, the non-private Winnow algorithm maintains as its hypothesis the halfspace $\sgn\langle w, \cdot \rangle$ where $w$ is an estimate of the unknown weight vector $v$. On examples where it predicts correctly, it reuses the $w$ it reported in the previous round. On examples $(x, y)$ where it predicts incorrectly, it multiplies each coordinate $w_j$ by a factor of $e^{\eta y x_j}$ for some learning rate $\eta$ so as to push $w$ in the direction of $v$ in terms of its prediction on $x$.

The learned weight vectors $w$ directly encode the given examples, so our private analog of Winnow releases randomized approximations of these instead. The first idea toward this comes from a beautiful observation from work using multiplicative weights updates for private query release~\citep{GaboardiGHRW14}: A random sample from a (suitably defined) multiplicative weights distribution may be viewed as an instantiation of the differentially private Exponential Mechanism~\citep{McSherryT07}. Thus, we can obtain a differentially private estimate of $w$ by approximating it from a small number of random samples. This idea was also used more recently by~\citet{asi2023nearoptimal} to give a private online experts algorithm.

There are two main challenges in turning this idea into a private online learner. The first is in ensuring that these approximate weight vectors are still suitable for learning. To argue that this is the case, we introduce a new (non-private) variant of the Winnow algorithm called $\mathsf{ConfidentWinnow}$ which is only guaranteed to stay with its current weight vector $w$ when its prediction is highly confident and correct, i.e., $y \langle w, x \rangle > c\rho$ and $\sgn \langle w, x\rangle = y$ for some constant $c$. These confident correct predictions -- ones where the inner product is large and has the right sign -- are robust to random sampling, in the sense that a small random sample from $w$ will, with high probability, yield the same prediction. We show that even in the face of an adversary who selects the rounds in which $\mathsf{ConfidentWinnow}$ updates its weights when it fails to make confident predictions, this algorithm updates not much more often than the original Winnow algorithm would have. This observation that we can demand the Winnow algorithm to make confident predictions without much overhead may be useful in contexts outside of differential privacy.

The second challenge is in ensuring privacy of the identities of the rounds in which an update occurs. That is, it is possible for neighboring sequences of examples to yield very different sequences of update rounds. To handle this, instead of deterministically updating in each round in which a prediction error occurs, we use the Sparse Vector technique~\citep{DworkR14} to update only after a small, privately identifiable number of errors has occurred. The resulting algorithm (Algorithm~\ref{alg: DP-WinnowHalfSpace}), which interleaves applications of the Sparse Vector technique with the sampling procedure described above while maintaining private internal state, is unwieldy to perform a privacy analysis of directly. Fortunately, since it is possible to reconstruct the private weight vector from just the observed examples and the publicly released predictions, one can analyze the algorithm as though it were just a sequential composition of these standard private components.

\paragraph{Related Work.} To our knowledge, our work is the first to explicitly study the problems of privately learning large-margin halfspaces or decision lists in the online model. Nevertheless, one could replace the $0$-$1$ loss with a convex surrogate (e.g., hinge loss) to make these problems amenable to general techniques for private online convex optimization~\citep{JainKT12, SmithT13, JainT14, AgarwalS17, Kairouz21b, AsiFKT23, asi2023nearoptimal}. These techniques, however, incur regret bounds that scale polynomially in the dimension $d$, whereas our target is to design algorithms with only a polylogarithmic dependence. Moreover, most early work on private online learning focused on the general non-realizable case, where a polynomial dependence on $T$ is necessary, and which cannot immediately be improved to $\operatorname{polylog}(T)$ under a realizability assumption.

Several more recent papers have studied the realizable setting more specifically. The aforementioned work of~\citet{GolowichLivni2021} showed that private online classifiability, in the realizable setting, is characterized by the Littlestone dimension of the concept class. Recent work of~\citet{Lyu25} showed that the dependence of the mistake bound on the Littlestone dimension can be improved significantly, from doubly exponential to polynomial, while preserving the logarithmic dependence on the time horizon $T$. Complementary work of~\citet{LiWY25a} developed improved private online learning algorithms for Littlestone classes against adaptive adversaries, and in the agnostic learning setting. As for lower bounds,~\citet{Cohen24} proved that mistake bound $\Omega(\log T)$ is necessary to privately learn point functions;~\cite{Dmitriev24} proved a similar result against a restricted class of learners. \citet{LiWY25b} established lower bounds for more general concept classes under pure differential privacy and under approximate differential privacy with small $\delta$.


\citet{asi2023nearoptimal} designed a private algorithm for learning from $d$ experts incurring $\operatorname{polylog}(d, T)$ regret under the realizability assumption that there is a zero-loss expert. They extended this to realizable online convex optimization by building a cover over the concept class, and interpreting each element of the cover as an expert. This immediately gives a $\operatorname{poly}(k, \log d, \log T)$-regret algorithm for the problem of learning $k$-sparse disjunctions described above, but does not give polylogarithmic regret for general halfspaces, as the size of the required cover is too large. 

The multiplicative weights technique is pervasive in differentially private algorithm design. In particular, the private multiplicative weights update is prevalent in the design of query release and synthetic data generation algorithms~\citep{HardtR10, HardtLM12, HsuRU13, GaboardiAHRW14,  GhaziGKKKMS25} and differentially private boosting for PAC learning~\citep{DworkRV10, BunCS20}.


\selectlanguage{vietnamese}
Finally, dimension-independent algorithms for private learning large-margin halfspaces in the PAC model are known via dimensionality reduction~\citep{NguyenUZ20} and boosting~\citep{BunCS20}, though these results do not have direct consequences for the online setting.

\selectlanguage{english}


\section{Preliminaries}

\paragraph{Decision Lists.} A monotone decision list of length $r$ is sequence of the form $(\ell_1, b_1), (\ell_2, b_2), \dots \\ (\ell_{r-1}, b_{r-1}), b_r$ where each $\ell_i$ is a variable and each $b_i \in \{-1, 1\}$. It computes the function ``if $\ell_1$ then $b_1$, else if $\ell_2$ then $b_2$, $\dots$, else $b_r$.'' For a class of Boolean functions $\mathcal{F}$, an $\mathcal{F}$-decision list is defined in the same way except each $\ell_i$ may be a function in $\mathcal{F}$. If the constant ``true'' function $T$ is included in $\mathcal{F}$, we can replace the final ``else'' rule with $(T, b_r)$ to allow for more concise notation.

\paragraph{Large-Margin Halfspaces.}
For a margin parameter $\rho > 0$, a $\rho$-margin halfspace is specified by a weight vector $v \in \mathbb{R}^d$ with $\|v\|_1=1$. It computes the function $\operatorname{sgn}(\langle v, x\rangle)$ and satisfies $|\langle v, x\rangle| \ge \rho$ for all $x \in \{-1, 1\}^d$.

\paragraph{Exponential Mechanism.} Let $q : \mathcal{Z}^n \times H \to \mathbb{R}$ be a sensitivity-1 score function, in the sense that $|q(S, h) - q(S', h)| \le 1$ for all neigboring datasets $S, S'$. The Exponential Mechanism $\mathcal{M}_E(S, q, \varepsilon)$ of~\citet{McSherryT07} identifies an $h \in H$ which approximately maximizes $q(S, \cdot)$ by sampling each $h$ with probability proportional to $\exp(-\varepsilon q(S, h)/2)$.
\begin{proposition}\label{proposition: Exponential Mechanism Utility} The Exponential Mechanism $\mathcal{M}_{E}$ is $\varepsilon$-differentially private. Moreover, for every $\tau > 0$, the Exponential Mechanism 
 outputs a solution $h$ that satisfies $q(S,h) \ge \max_{f\in H}(q(S,f)) - \frac{2}{\varepsilon}(\log(|H|)+ \tau)$ with probability at least $1-e^{-\tau}$ for any deviation term $\tau>0$.
\end{proposition}

\section{Private PAC Learner for Decision Lists}

In this section, we present $\mathsf{DP\text{-}GreedyCover}$, a computationally and sample efficient PAC learner for decision lists. We describe our algorithm as directly handling $\mathcal{F}$-decision lists for some class of feature functions $\mathcal{F}$. Our algorithm will be differentially private and an accurate PAC learner according to the following definitions.

\begin{definition}[\cite{DworkMNS06, KasiviswanathanLNRS08}]
    A learning algorithm $L: (\mathcal{X} \times \{0, 1\})^n$ is $(\varepsilon, \delta)$-differentially private if, for all neighboring datasets $S, S' \in (\mathcal{X} \times \{0, 1\})^n$ (i.e., differing in a single example) and all sets of hypotheses $R$,
    \[\Pr[L(S) \in R] \le e^{\varepsilon} \Pr[L(S')\in R] + \delta.\]
\end{definition}

\begin{definition}
    Let $\mathcal{D}$ be a distribution over a domain $\mathcal{X}$, and let $\mathcal{C}$ be a concept class of Boolean functions over $\mathcal{X}$. Given $c, h : \mathcal{X} \to \{0, 1\}$, define
    \[err_{\mathcal{D}}(c, h) = \Pr_{x \sim \mathcal{D}}[c(x) \ne h(x)].\]
    An algorithm $L$ is an $(\alpha, \beta)$-PAC learner with sample complexity $n$ if, for every distribution $\mathcal{D}$ and every target concept $c \in \mathcal{C}$,
    \[\Pr_{h \gets L(S)}[err_{\mathcal{D}}(c, h) \le \alpha]\ge 1 - \beta,\]
    where the probability is taken over $S = \{(x_i, c(x_i)\}_{i=1}^n$ where the $x_i$-s are drawn i.i.d. from $\mathcal{D}$, as well as any internal randomness of $L$.
\end{definition}

\subsection{Algorithm Overview}

$\mathsf{DP\text{-}GreedyCover}$ operates on a sample $S$ consisting of samples of the form ($x_i, \sigma_i)$, where each $\sigma_i = c^*(x_i)$ for some unknown decision list $c^*$. The algorithm iteratively maintains a set $S_j$ of unclassified samples and applies the Exponential Mechanism to select a function-label pair $(f_j,b_j)$ with high score with respect to a quality function $q(S_j, (f_j, b_j))$ that guarantees that $(f_j, b_j)$ errs on few examples in $S_j$. Then, the algorithm continues to the next iteration with $S_{j+1}$, the subset of the sample $S_j$ that contains examples on which $f_j$ evaluates to 0, and $\mathcal{F}_{j+1}$, the subset of $\mathcal{F}_{j}$ without $f_j$.

\begin{algorithm2e}[H]
\caption{$\mathsf{DP}\text{-}\mathsf{GreedyCover}$}
\label{alg: DP-GreedyCover}
\SetAlgoLined
\LinesNumbered
\DontPrintSemicolon
\KwIn{ Labeled sample $S = \{ (x_i, \sigma_i)\}_{i=1}^n \in (\{0,1\}^d \times  \{ 0, 1\} )^n$, set of feature functions $\mathcal{F}$ }
\KwOut{Decision list $h_{fin} = [(f_1, b_1), (f_2, b_2), ..., (f_M, b_M) ]$}
Initialize: Let $\mathcal{F}_1 = \mathcal{F}$ together with the constant true function $T$, and let $S_1 = S$.\;
\For{$j = 1, 2, \dots, M$}{
        Let $S_{j}^1$ and $S_{j}^0$ denote the  positive and negative examples in $S_{j}$, respectively.\;
        \For{$f$ $\in$ $\mathcal{F}_j$}{
        Let $\underset{f \rightarrow 1}{\#}(S_{j}^1)$ and $\underset{f \rightarrow 1}{\#}(S_{j}^0)$ denote the number of positive and
        negative examples in $S_j$, respectively, that function $f$ labels as 1. That is,\;
        {\hspace{10mm} $\underset{f \rightarrow 1}{\#}(S_{j}^1) = |\{ x \in S_{j}^1 : f(x) = 1\}|$}\;
        {\hspace{10mm} $\underset{f \rightarrow 1}{\#}(S_{j}^0) = |\{ x \in S_{j}^0 : f(x) = 1\}|$      \;
        }
        } 
        \For{$f$ $\in$ $\mathcal{F}_j$, $b \in \{0,1\}$}{
        Define the quality function $q(S_j, (f , b)) =  -\#_{f \rightarrow 1}(S_{j}^{1-b})$\;
        } 
        Let $(f_j,b_j) \leftarrow \mathcal{M}_E(S_j, q, \varepsilon)$\;
        Let $h_j \leftarrow (f_{j}, b_{j})$, $S_{j+1} = \{x \in S_{j} | f_{j}(x) = 0\}$, and $\mathcal{F}_{j+1} = \mathcal{F}_j - \{f_j\}$\;
} 
Return the hypothesis $h_{fin} = [h_1, h_2, ..., h_{M}]$\;
\end{algorithm2e}

The success of the iterative selection procedure depends on the answer to the following question: Is there always a candidate function within $ \mathcal{F}_{j} $ capable of accurately learning the remaining samples, especially as the subsets $S_j$ become progressively smaller?

We show that this is the case in Proposition~\ref{proposition:consistent k-decision list}, which confirms that decision lists can be learned ``consistently'' by a covering algorithm -- any remaining subset of the sample $S$ labeled by a known decision list $c^*$ can indeed be perfectly learned by a decision list comprised of the candidate functions. At the end of iteration $M$, the algorithm will have exhausted all available functions and all examples will have been classified.

\begin{proposition} \label{proposition:consistent k-decision list}
Let $S = \{(x_i, \sigma_i)\}_{i = 1}^{n} \in (\{0,1\}^d \times \{0, 1\})^n$ be a sample that is labeled by a decision list $c^*$ in the general form of $[(f_1, b_1), (f_2, b_2), ..., (f_m, b_m), (T, b_{m+1})] $ where $m \leq M$. For every iteration of running the $\mathsf{DP\text{-}GreedyCover}$ over $S$, there exists a decision list $c_j$ over the remaining set of candidate functions $\mathcal{F}_{j}$ that is consistent with the remaining sample $S_{j}$. 
\end{proposition}

\begin{proof} 
It suffices to show that by induction over iteration $j$ of $\mathsf{DP\text{-}GreedyCover}$'s execution over $S$, there exists a $c_j$ in $\mathcal{F}_{j}$ which is consistent with the sample $S_j$.

For $j = 1$, we set $c_j = c^*$. For $j \ge 2$, suppose that there exists a $c_j$ over $\mathcal{F}_{j}$ that is consistent with the sample $S_j$ and fix an arbitrary $f_z \in \mathcal{F}_{j}$. 
Consider the following case analysis: 

\begin{enumerate}
    \item  Assume $f_z$ does not appear in $c_j$. Then we may take $c_{j+1} = c_j$. To see this, since $S_{j}' \subseteq S_{j}$ and $c_j$ is a decision list over $\mathcal{F}_{j}\textbackslash \{f_z\}$ that is consistent with $S_{j}$, $c_{j}$ is also consistent with $S_{j}'$.
    \item Assume $f_z$ appears in $c_j = [(f_1, b_1), (f_2, b_2), ..., (f_z, b_z), ..., (f_m, b_m), (T, b_{m+1})]$. Then we may take $c_{j+1} = [(f_1, b_1), ...(f_{z-1}, b_{z-1}),(f_{z+1}, b_{z+1}), ..., (T, b_{m+1})]$, which is $c_j$ without the term $(f_z, b_z)$. 
\end{enumerate}

Observe that the subset $S_j$ that has been classified up to rule $f_{z-1}$ is the same as the subset $S_{j}'$ that has been classified up to function $f_{z-1}$. Since $S_{j} = S_{j}'$, for the terms $(f_{z+1}, b_{z+1})$ and after, the residual subset $S_{j}$ is classified in the same pattern as $S_{j}'$.
\end{proof}

\subsection{Algorithm Analysis}

Theorem~\ref{theorem: zCDP DP-GreedyCover} summarizes the privacy and accuracy guarantees of $\mathsf{DP\text{-}GreedyCover}$. The complete proof is provided in Appendix~\ref{appendix:greedy-cover}.

\begin{theorem}\label{theorem: zCDP DP-GreedyCover}
Let $M = |\mathcal{F}|$. There exists a $(\varepsilon, \delta)$-DP $(\alpha,\beta)$-PAC learner for the concept class $\mathcal{C}$ of $\mathcal{F}$-decision lists with sample complexity 
\[ 
n \geq \max\left(\frac{64}{\alpha}(VC(\mathcal{C})\log(\frac{64}{\alpha}) + \log(\frac{16}{\beta})), \frac{8M\log(\frac{2M}{\sqrt{\beta}})(2\log(\frac{1}{\delta})+\frac{3}{2})}{\alpha\varepsilon}\right)
\]
and runtime $\operatorname{poly}(M)$.
\end{theorem}
Note that the VC dimension of $\mathcal{C}$ is at most the logarithm of its cardinality, which is also $\tilde{O}(M)$.\\

Our algorithm makes essential use of realizability. Even non-privately, the prospect of extending these techniques to the agnostic setting faces fundamental computational barriers~\cite{FeldmanGRW09}. Specifically, efficient agnostic learning of even disjunctions would imply a breakthrough result for PAC learning poly-size DNF, suggesting very different techniques will be needed.

\section{Differentially Private Winnow for Online Learning Halfspaces}

 Recall that we model online learning as a sequential $T$-round game between a learner and an adversary. In this paper, we focus on \emph{oblivious} adversaries, which must choose a sequence of labeled examples $(x_1, y_1), \dots, (x_T, y_T)$ in advance, but presents them to the learner in an online fashion. In each round $t$ of the game, the adversary reveals unlabeled example $x_t$, the learner responds with a hypothesis $h_t : \mathcal{X} \to \{-1, 1\}$, and then the adversary reveals the label $y_t$.

\begin{definition}[Differentially Private Learning Against an Oblivious Adversary]
Let $S$ denote a sequence of labeled examples of the form $(x_1, y_1), \dots, (x_T, y_T)$. Denote by $M(S)$ the sequence $h_1, \dots, h_T$ of hypotheses produced by an online learner $M$ playing the above game against an adversary choosing the sequence $S$. We say that learner $M$ is $(\varepsilon, \delta)$-differentially private if, for every pair $S, S'$ differing in a single example and every set $R$ of sequences of hypotheses, we have $\Pr[M(S) \in R] \le e^{\varepsilon}\Pr[M(S') \in R] + \delta$.
\end{definition}

\subsection{ConfidentWinnow}
\vspace{-1mm}
In this section, we present $\mathsf{ConfidentWinnow}$, a non-private online learning algorithm for learning large-margin halfspaces. This algorithm serves as the foundational subroutine for its private analog, detailed in Section \ref{subsection: DP-Winnow}. The pseudocode of the algorithm is given in Algorithm~\ref{alg: ConfidentWinnowHalfSpace}. Unlike the traditional Winnow Algorithm, $\mathsf{ConfidentWinnow}$ is guaranteed not to update its weight vector only when it predicts both correctly and confidently.

\begin{algorithm2e}
\label{alg: ConfidentWinnowHalfSpace}
\SetAlgoLined
\LinesNumbered
\DontPrintSemicolon
\KwIn{Time horizon $T$, confidence parameter $c$, margin parameter $\rho$}
\KwOut{Sequence of weight vectors ${w}^{(1)}, \ldots, {w}^{(T)}$ that represent halfspaces}
Initialize: $w^{(0)} = (\frac{1}{d}, \dots, \frac{1}{d})$\;
\For{$t = 1, 2, \dots, T$}{
        $w^{(t)} \gets w^{(t-1)}$\;
        Adversary presents example $x^{(t)}$\;
        Learner outputs $w^{(t)}$\;
        Adversary reveals $y^{(t)}$ and chooses bit $b^{(t)} \in \{0, 1\}$\;
        \If{$\sgn(\langle w^{(t)}, x^{(t)}\rangle) \ne y^{(t)}$ or ($-c\rho < \langle w^{(t)}, x^{(t)}\rangle < c\rho $ and $b^{(t)} = 1$)}{
        {\color{gray} // Perform multiplicative weight update on $w^{(t)}$ using $(x^{(t)}, y^{(t)})$}\;
        \For{$j=1$ to $d$}{
                 $w^{(t)}_{j} \gets \frac{\exp({\eta y^{(t)}x^{(t)}_j})}{\sum_{k=1}^{d}w^{(t-1)}_k \exp({\eta y^{(t)}x^{(t)}_k})}\cdot w^{(t-1)}_{j}$\;
        }   
    }     
}   
\caption{$\mathsf{ConfidentWinnow}$}
\end{algorithm2e}

Lemma~\ref{lemma:ConfidentWinnowUpdatesUpperBound} below summarizes the guarantees of $\mathsf{ConfidentWinnow}$. Our analysis closely follows the standard analysis of Winnow as presented by~\citet{Mohri2018} and appears in Appendix~\ref{appendix:confident-winnow}.
\begin{lemma}\label{lemma:ConfidentWinnowUpdatesUpperBound}
Let $x^{(1)}, \dots, x^{(T)} \in \{-1, 1\}^{d}$ be an arbitrary sequence of $T$ points and $b^{(1)}, \dots, b^{(T)}$ and arbitrary sequence of update instructions. Assume that there exists a weight vector $v \in \mathbb{R}^{d}$ where $v \ge 0$ and $\|v\|_1 = 1$ and a margin $\rho \ge 0$ such that for every $t \in [T]$, $y^{(t)}(v\cdot x^{(t)})\ge \rho$.
Then, when run with learning rate $\eta$ and confidence ratio $c < 1/2$, the number of updates made by $\mathsf{ConfidentWinnow}$ is at most $\frac{\log d}{(1-c)\eta\rho - \eta^2/2}$.
\end{lemma}
\subsection{DP-Winnow}\label{subsection: DP-Winnow}
In this section, we describe $\mathsf{DP\text{-}Winnow}$ (Algorithm~\ref{alg: DP-WinnowHalfSpace}), our main algorithm for learning large-margin halfspaces in the realizable setting with an oblivious adversary. 

\paragraph{Sparse Vector Technique.} Our differentially private Winnow algorithm relies on the sparse vector technique~\citep{DworkR14}. The technique rests on an interactive algorithm called $\mathsf{AboveThreshold}$ that, given a threshold parameter $L$ and a sequence of sensitivity-1 queries \\$q^{(1)}, q^{(2)}, \dots, q^{(T)}$ presented in a online fashion, approximately determines the first query whose value exceeds $L$. Following~\citet{asi2023nearoptimal}, we describe the interface to $\mathsf{AboveThreshold}$ via the following three functions:

\begin{enumerate}
\item $\mathsf{InitializeAboveThr}(\varepsilon, L, \beta)$: Initialize a new instance of $\mathsf{AboveThreshold}$ with privacy parameter $\varepsilon$, threshold $L$, and failure probability $\beta$. This returns a data structure $Q$ supporting the following two functions.
\item $Q.\mathsf{AddQuery}(q)$: Adds a new sensitivity-1 query $q: \mathcal{Z}^n \to \mathbb{R}$ to $Q$.
\item $Q.\mathsf{TestAboveThr}()$: Test if the most recent query added to $Q$ was (approximately) above the threshold $L$. If that is the case, the data structure ceases to accept further queries.
\end{enumerate}

\begin{lemma}~\citep{DworkR14} \label{lemma: SVT properties}
There is an $\varepsilon$-differentially private interactive algorithm $\mathsf{AboveThreshold}$ with the interface described above. For every threshold parameter $L$ and failure probability $\beta$, and every dataset $D \in \mathcal{Z}^n$, with probability at least $1-\beta$, the following holds. Let $q^{(1)}, \dots, q^{(T)}$ be the sequence of queries added to $Q$, and let $k$ be the index of the first query for which $Q.\mathsf{TestAboveThr()}$ would return $\mathsf{True}$ (or $T+1$ if no such query exists). Then for $\alpha = \frac{8 \log (2T/\beta)}{\varepsilon}$,
\begin{enumerate}
    \item For all $t < k$, we have $q^{(t)}(D) \le L + \alpha$, and
    \item Either $q^{(k)}(D) \ge L -\alpha$ or $k = T+1$.
\end{enumerate}
\end{lemma}

\paragraph{}{We leverage the sparse vector technique to efficiently and privately decide update rounds. In classical Winnow and $\mathsf{ConfidentWinnow}$, updates only occur on rounds in which a prediction error occurred or if the prediction is ``unconfident.'' By using $\mathsf{AboveThreshold}$ to conditionally trigger updates when errors exceed a noisy threshold}, we ensure the privacy cost scales with the number of updates ($\operatorname{polylog}(T)$) rather than the total number of rounds $T$.

\begin{algorithm2e}[H]
\caption{$\mathsf{DP}\text{-}\mathsf{Winnow}$}
\label{alg: DP-WinnowHalfSpace}
\SetAlgoLined
\LinesNumbered
\DontPrintSemicolon
\KwIn{Time horizon $T$, margin parameter $\rho$, privacy 
 parameter $\hat{\varepsilon}$, learning rate $\eta$, threshold parameter $L$, normalization parameter $m$, failure probability $\beta$, switching bound $K$,  Examples $(x^{(1)}, y^{(1)}), \dots, (x^{(T)}, y^{(T)})$ }
\KwOut{Sequence of weight vectors $\tilde{w}^{(1)}, \ldots, \tilde{w}^{(T)}$ representing halfspaces}
Initialize: $k \gets 0$, $t_p \gets 0$, $C \gets \emptyset$, 
$w^{(1)} \gets  \left(\frac{1}{d}, \frac{1}{d}, \ldots, \frac{1}{d}\right)$,
$\tilde{w}^{(1)} \gets  \left(\frac{1}{d}, \frac{1}{d}, \ldots, \frac{1}{d}\right)$\;
$Q \gets \mathsf{InitializeAboveThr}(\hat{\varepsilon}, L, \beta/T)$\;
\For{$t = 1, 2, \dots, T$}{
    Receive unlabeled example $x^{(t)}$\;
    Output hypothesis $\tilde{w}^{(t)}$\;
    Receive label $y^{(t)}$\;
    \If{$\sgn(\langle\tilde{w}^{(t)}, x^{(t)}\rangle) \ne y^{(t)}$} {
    $C \gets C\cup \{(x^{(t)}, y^{(t)})\}$\;
    }
    Define query $q^{(t)} \gets \sum_{z = t_p}^{t} 1[\sgn(\langle\tilde{w}^{(z)}, x^{(z)}\rangle) \ne y^{(z)}]$\;
    $Q.\mathsf{AddQuery}(q^{(t)})$\;
    \If{$Q.\mathsf{TestAboveThr()}$ and $k < K$} {\label{alg:DP-Winnowtest}
         Let $(\hat{x},\hat{y})$ be the first example cached in $C$ \label{alg:DP-Winnow-weights-update-start}\;
        {\color{gray} // Perform multiplicative weight update on $w^{(t)}$ using $(\hat{x}, \hat{y})$}\;
        \For{$j = 1, 2, \dots, d$}{
                  $w^{(t)}_{j} \gets \frac{\exp({\eta \hat{y}\hat{x}_j})}{\sum_{k =1}^{d}w^{(t-1)}_k \exp({\eta \hat{y}\hat{x}_k})}\cdot w^{(t-1)}_{j}$\;
        } 
        {\color{gray} // Update approximate weight vector $\tilde{w}^{(t)}$ by sampling from $w^{(t)}$}\;
         $j_1, \ldots, j_m \gets [d]$ i.i.d. where for each $i$, $\Pr[j_i = j] = w^{(t)}_j$\; \label{alg:DP-Winnow-weights-sample}
        \For{$j = 1, 2, \dots, d$}{
            $\tilde{w}^{(t)}_j \gets \frac{1}{m} \cdot \#\{i : j_i = j\}$\;
        } \label{alg:DP-Winnow-weights-update-end} 
         Update $k \gets k+1$, $C \gets \emptyset$, $t_p \gets t+1$\;
        $Q \gets \mathsf{InitializeAboveThr}(\hat{\varepsilon}, L, \beta/T)$\;
    } 
} 
\end{algorithm2e}

\subsubsection{Privacy Analysis} \label{sec:winnow-privacy}
\vspace{-1mm}

\begin{theorem} \label{thm:winnow-privacy}
    Let $\varepsilon, \delta \in (0, 1)$. If $\hat{\varepsilon} = \varepsilon/(4\sqrt{2K \log (2/\delta)})$ and $\eta = \varepsilon/(8\sqrt{2mK\log(2/\delta)})$, then algorithm $\mathsf{DP\text{-}Winnow}$ is $(\varepsilon, \delta)$-differentially private.
\end{theorem}

Our privacy proof appears in Appendix~\ref{appendix:online-privacy}.

\subsubsection{Regret Analysis} \label{sec:winnow-regret}

In this section, we summarize the performance of $\mathsf{DP\text{-}Winnow}$ (Algorithm~\ref{alg: DP-WinnowHalfSpace}) via Theorem~\ref{theorem: DP-Winnow privacy and regret}.

\begin{theorem} \label{theorem: DP-Winnow privacy and regret}
When invoked using parameters $K = \frac{2\log d}{\eta\rho-\eta^2}$, $\hat{\varepsilon} = \varepsilon/(4\sqrt{2K \log (2/\delta)})$,\\ 
$\eta = \varepsilon/(8\sqrt{2mK\log(4K/\delta)})$, $L = \frac{8 \log (2T/\beta)}{\hat{\varepsilon}}$ and $m = 2\log(2T/\beta)/\rho^2$, Algorithm $\mathsf{DP\text{-}Winnow}$ is $(\varepsilon, \delta)$-differentially private and makes at most
\[\tilde{O}\left(\frac{\log^3 d \cdot \log^{5/2}(T/\beta) \cdot \log^2(1/\delta)}{\varepsilon^4 \rho^6}\right)\]
prediction errors with probability at least $1-2\beta$.
\end{theorem}

We analyze the regret bound stated in Theorem~\ref{theorem: DP-Winnow privacy and regret} as follows.

\begin{enumerate}
    \item \textbf{Regret Bound of $\mathsf{ConfidentWinnow}$}: 
    We previously showed that $\mathsf{ConfidentWinnow}$ learns halfspaces with finite mistake bound as detailed in Lemma~\ref{lemma:ConfidentWinnowUpdatesUpperBound}.     
    \item \textbf{Consistency of Weight Vectors}: 
    We demonstrate in Lemma~\ref{lemma: high probability argument matching shadow and approximate weight vectors} that, with high probability, the shadow weight vectors and approximate weight vectors used by $\mathsf{ConfidentWinnow}$ make equivalent predictions when the predictions are confident. 
    \item \textbf{Equivalence of Regret and Update Bounds}: We show that the number of mistakes made by $\mathsf{ConfidentWinnow}$ is asymptotically equivalent to $\mathsf{DP\text{-}Winnow}$'s update bound in Lemma~\ref{lemma: DP-Winnow wSwitchingBudget and ConfidentWinnow Updates Equivalence}, bridging the gap between the non-private and private learning algorithms. 
    \item \textbf{Cost Efficiency of the Sparse Vector Technique}: By applying the properties of the Sparse Vector Technique described in Lemma~\ref{lemma: SVT properties}, we establish that $\mathsf{DP\text{-}Winnow}$, with high probability, incurs a polylogarithmic number of mistakes per update.
\end{enumerate}

 Lemma~\ref{lemma: high probability argument matching shadow and approximate weight vectors} demonstrates that with high probability, the predictions made by the shadow weight vectors $w^{(t)}$ and the approximate weight vectors $\tilde{w}^{(t)}$ are consistent across all rounds $t \in [T]$.
 
\begin{lemma}\label{lemma: high probability argument matching shadow and approximate weight vectors}
Let $x \in \{-1, 1\}^d$ be an unlabeled example and let $w$ be a weight vector with $\|w\|_1 = 1$ such that $|\langle w, x\rangle| \ge c\rho$. Construct $\tilde{w}$ according to the following process:
\begin{itemize}
    \item Define a probability distribution over the index set $[d]$ where each coordinate $j$ is sampled with probability $w_{j}$.
    \item For each $i \in \{1, \dots, m\}$, independently sample indices $j_i$ from this distribution.
    \item Set $\tilde{w}_{j}$ for each $j \in [d]$ to $ \frac{1}{m}$ multiplied by the count of $j_i$'s equal to $j$.
\end{itemize}
Let $\beta > 0$ and $m\ge \frac{2}{c^2\rho^2}\log(\frac{2T}{\beta})$. Then, with probability at least $1-\beta/T$,
\[
\operatorname{sgn}(\langle \tilde{w}, x \rangle) = \operatorname{sgn}(\langle w, x \rangle).
\]
\end{lemma}

The proof appears in Appendix~\ref{appendix:winnow-match}.

Lemma~\ref{lemma: DP-Winnow wSwitchingBudget and ConfidentWinnow Updates Equivalence} below shows that that the upper bound on the number of updates required by $\mathsf{ConfidentWinnow}$ is asymptotically equivalent to the switching bound of $\mathsf{DP\text{-}Winnow}$.

\begin{lemma}\label{lemma: DP-Winnow wSwitchingBudget and ConfidentWinnow Updates Equivalence} 
 For any (obliviously chosen) sequence of examples consistent with a $\rho$-margin halfspace, with probability at least $1-\beta$, algorithm $\mathsf{DP\text{-}Winnow}$ performs fewer than $\log d / ((1-c)\eta\rho - \eta^2/2)$ updates. Thus, if we set $K = \log d / ((1-c)\eta\rho - \eta^2/2)$, the condition ``$k < K$'' in line~\ref{alg:DP-Winnowtest} will never be violated.
\end{lemma}

\begin{proof}

Fix a sequence of examples consistent with a $\rho$-margin halfspace. We compare a run of $\mathsf{DP\text{-}Winnow}$ (denoted here by $\mathcal{A}_1$) with a run of $\mathsf{ConfidentWinnow}$ (denoted here by $\mathcal{A}_2$) on the subsequence of examples on which $\mathcal{A}_1$ performs an update. Simplifying our notation to index these examples, let us denote them by $(x^{(1)}, y^{(1)}), \dots, (x^{(B)}, y^{(B)})$. 

By construction, these are only candidates for updates because each $\sgn(\langle \tilde{w}^{(t)}, x^{(t)}\rangle) \ne y^{(t)}$. By Lemma~\ref{lemma: high probability argument matching shadow and approximate weight vectors}, with probability at least $1-\beta$, we must have $\langle w^{(t)}, x^{(t)} \rangle \cdot y^{(t)} \le c\rho$ for all of these candidates. Thus, if $\mathcal{A}_2$ is instructed to update its weights when presented with each of these examples, it maintains the same sequence of weights $w^{(1)}, \dots, w^{(B)}$ as $\mathcal{A}_1$.  By Lemma~\ref{lemma:ConfidentWinnowUpdatesUpperBound}, the number of such updates is upper bounded by the stated quantity.
\end{proof}

\begin{proof}[Proof of Theorem~\ref{theorem: DP-Winnow privacy and regret}]\label{proof: DP-Winnow privacy and regret}

By Lemma~\ref{lemma: DP-Winnow wSwitchingBudget and ConfidentWinnow Updates Equivalence}, with probability $1-\beta$, Algorithm $\mathsf{DP\text{-}Winnow}$ performs at most $K = \frac{\log d}{(1-c)\eta\rho-\eta^2/2}$ updates, and only incurs prediction errors between those updates. The number of such prediction errors between updates is governed by the accuracy of $\mathsf{AboveThreshold}$ which, by Lemma~\ref{lemma: SVT properties}, guarantees at most $\frac{16\log(2T^2/\beta)}{\hat{\varepsilon}}$ prediction errors per update with total probability $1-\beta$. Hence, we have that with probability $1-2\beta$, the total number of prediction errors is at most $\frac{16\log(2T^2/\beta)}{\hat{\varepsilon}} \cdot K$.

By Theorem~\ref{thm:winnow-privacy}, it suffices to take $\hat{\varepsilon} = \varepsilon/(4\sqrt{2K \log (2/\delta)})$ and $\eta = \varepsilon/(8\sqrt{2mK\log(4K/\delta)})$ to guarantee $(\varepsilon, \delta)$-differential privacy overall. Fix $c = 1/2$ and recall that we are taking $m = 2\log(2T/\beta)/c^2\rho^2$, so $\eta < \rho / 2$. Hence, $K \le \frac{4\log d}{\eta \rho}$. Plugging in our setting of $\eta$ gives 
\[K \le \tilde{O} \left( \frac{(\log^2 d) m \log(1/\delta)}{\varepsilon^2 \rho^2}\right).\]
Thus, with probability at least $1-2\beta$, the total number of prediction errors is at most
\[\tilde{O}\left(\frac{\log^3 d \cdot \log^{5/2}(T/\beta) \cdot \log^2(1/\delta)}{\varepsilon^4 \rho^6}\right).\]
\end{proof}

While our analysis makes essential use of the realizability assumption, as does the classic Winnow algorithm, an interesting direction for future work would be to extend it to learn a drifting target as studied, e.g., by~\cite{Blum1998}. 

\subsection{Application: Privately Learning Decision Lists}

In this section, we explain how our Algorithm \ref{alg: DP-WinnowHalfSpace} for learning large-margin halfspaces yields a private algorithm for learning decision lists. This is based on a classic connection observed in work of~\citet{blumandsingh1990},~\citet{dhagat1994pac},~\citet{1999_Valiant}, and~\citet{NevoE03}.

In general, if a decision list has length $r$ and its output bits alternate between $+1$ and $-1$ at most $D$ times, then it is represented by a halfspace with large margin.

\begin{lemma}~\citep{blumandsingh1990, dhagat1994pac, 1999_Valiant}\label{lemma: decision-lists represent ltf}
    Every monotone 1-decision list over $d$ variables with length $r$ and $D$ alternations can be represented by a halfspace over $d+1$ variables with margin $\Omega(1/r^D)$.
\end{lemma}

However, the weight vector for the resulting halfspace might have negative entries. For our $\mathsf{DP\text{-}Winnow}$ algorithm to apply, we apply a simple transformation to convert a halfspace over $d$ dimensions (with possibly negative entries in its weight vector) into one over $2d$ dimensions $z_1, \dots, z_{2d}$ with only nonnegative entries by setting

\[
w'_i = \begin{cases}
|w_i| &  \text{if } i \le d, w_i \ge 0 \\
|w_{i-d}| &  \text{if } i > d, w_i < 0\\
0 & \text{otherwise}.
\end{cases}
\qquad
z_i = \begin{cases}
x_i &  \text{if } i \le d\\
-x_{i-d} & \text{otherwise}.
\end{cases}
\]

Note that the resulting halfspace has the same margin as the original halfspace. 

\begin{theorem}
Algorithm $\mathsf{DP\text{-}Winnow}$ learns the class of length-$r$ monotone 1-decision lists with $D$ alternations while making at most
 \[\tilde{O}\left(\frac{r^{6D} \cdot \log^3 d \cdot \log^{5/2}(T/\beta) \cdot \log^2(1/\delta)}{\varepsilon^4}\right)\]
 prediction errors with probability at least $1-O(\beta)$.
\end{theorem}

In general, if instead of taking each $\ell_i$ to be a variable, one takes it to be a member of some feature space $\mathcal{F}$, then we can use the same algorithm to obtain

\begin{theorem}
Algorithm $\mathsf{DP\text{-}Winnow}$ learns the class of length-$r$ $\mathcal{F}$-decision lists with $D$ alternations while making at most
 \[\tilde{O}\left(\frac{r^{6D} \cdot \log^3 |\mathcal{F}| \cdot \log^{5/2}(T/\beta) \cdot \log^2(1/\delta)}{\varepsilon^4}\right)\]
 prediction errors with probability at least $1-O(\beta)$.
\end{theorem}

\subsection{Lower Bound} \label{section:lower}

Let $\mathcal{P} = \{p_t : [T] \to \{-1, 1\}\}$ be the class of point functions over a domain $[T]$ defined by $p_t(x) = 1 \iff x = t$.~\citet{Cohen24} showed that every $(\varepsilon = 1/2, \delta = 1/(20 \log T))$-differentially private online learner for $\mathcal{P}$ incurs at least $\Omega(\log T)$ mistakes on some sequence of $T$ examples.

We now argue that this implies a lower bound of $\Omega(\min\{\log d, \log T\})$ for privately learning halfspaces with margin $\rho = 1$. In particular, this implies that for $T \ll d$, a polylogarithmic dependence on $T$ as in the statement of Theorem~\ref{theorem: DP-Winnow privacy and regret} is necessary. To see this, observe that we can embed point functions into the class of margin-$1$ halfspaces in dimension $d = T$ by mapping each domain point $x \in [T]$ to the point $(1, 1, \dots, 1, -1, 1, \dots, 1)$, where the $-1$ appears in the $x$'th coordinate, and mapping each function $p_t$ to the halfspace whose weight vector is the $t$'th standard basis vector $(0, 0, \dots, 0, 1, 0, \dots 0)$.

\acks{MB was supported by NSF CNS-2046425. WF thanks Debanuj Nayak and Satchit Sivakumar for helpful discussions.

}

\bibliography{referencesalt}

\appendix


\section{Proof of Theorem~\ref{theorem: zCDP DP-GreedyCover}}
\label{appendix:greedy-cover}

Our analysis underlying Theorem~\ref{theorem: zCDP DP-GreedyCover} goes by way of the following roadmap.

\begin{enumerate}
    \item \textbf{$\mathsf{EMCover}$ Framework}: We introduce $\mathsf{EMCover}$, a generic learning framework that $\mathsf{DP\text{-}GreedyCover}$ instantiates. $\mathsf{EMCover}$ employs the Exponential Mechanism to iteratively select hypotheses, with each selection dependent on previous rounds. We state the precise privacy guarantees in Lemma \ref{lemma:Quality constraint implies DP}. The pseudocode of $\mathsf{EMCover}$ is shown in Algorithm~\ref{alg:EMCoverLearner}. 
    
    \item \textbf{Martingale Argument for Privacy Analysis}: To prove the privacy claim established in Lemma~\ref{lemma:Quality constraint implies DP}, we invoke the martingale argument of Lemma~\ref{lemma:random-process} (detailed in~\citet{gupta2009differentially, DBLP:journals/corr/abs-1902-05017}).
    
    \item \textbf{Empirical Error of $\mathsf{DP\text{-}GreedyCover}$}: In Lemma~\ref{lemma: DP-GreedyCover Empirical err}, we show that $\mathsf{DP\text{-}GreedyCover}$ does not incur too much error with respect to its given sample.
    
    \item \textbf{Empirical Error to Generalization Error}: Theorem~\ref{theorem: VC Dimension Generalization Bound} shows that when given a sample of size above the VC dimension, our algorithm's small empirical error translates into small generalization error. 
\end{enumerate}

\subsection{Privacy Analysis}

We consider the following general framework for covering algorithms based on the Exponential Mechanism. This framework abstracts and generalizes a presentation of~\citet{DBLP:journals/corr/abs-1902-05017}, who studied a similar algorithm for learning sparse disjunctions and conjunctions. Let $Z$ be a data universe, let $H$ be a set of candidates, and let $Q$ be a family of score functions of the form $q: Z^* \times H \rightarrow \mathbb{R}$. For each $z \in Z$, let $H_z$ be the subset of $H$ of candidates that ``covers'' data point $z$. We are interested in EM-based cover algorithms that take the following form:

\begin{algorithm2e}[H]
\caption{$\mathsf{EMCover}$}
\label{alg:EMCoverLearner}
\SetAlgoLined
\LinesNumbered
\DontPrintSemicolon
\KwIn{ Dataset $S \in Z^*$, privacy parameter $\varepsilon$ }
\KwOut{List of functions $h_1, \dots, h_T$}
Initialize: Let $\mathcal{F}_1 = \mathcal{F}$ together with the constant true \;
\For{$t = 1, \dots, T$}{
        Select a score function $q_t$ from $Q$ (arbitrarily, and possibly adaptively, based on the outcomes of previous iterations)\;
        Invoke the Exponential Mechanism to obtain $h_t \leftarrow \mathcal{M}_E(q_t, S, \varepsilon)$\;
        Update $S$ by removing all elements $z$ that are covered by $h_t$\;
}
\end{algorithm2e}

When $Q$ is structured, this algorithm gives privacy guarantees that are much stronger than what one might expect based on standard composition theorems. 

\begin{lemma}\label{lemma:Quality constraint implies DP}
If the family of quality functions $Q$ has the property that for every $q \in Q$, every $h \in H$, and every $z \in Z$:

\begin{enumerate}
\item If $h$ covers $z$, then for every dataset $S$, we have  $q(S, h) - 1 \le q(S \cup \{z\}, h) \le q(S, h) + 1$, and
\item If $h$ does not cover $z$, then for every dataset $S$, we have  $q(S \cup \{z\}, h) = q(S, h)$,
\end{enumerate}
then for every $\delta > 0$, algorithm $\mathsf{EMCover}$ is $ (2\varepsilon(\ln(1/\delta) + \frac{3}{2}), \delta)$-differentially private. 
\end{lemma}

\begin{proof}
Let $S$ be any dataset and let $S' = S \cup \{z\}$ for some $z \in Z$. Fix a sequence of candidate outputs $h_1, \dots, h_T$. Let $t$ be the smallest index for which $h_t$ covers $z$ (or $t = T$ otherwise). 

Case (a)
\begin{align*}
\frac{\Pr[\mathsf{EMCover}(S) = (h_1, \dots, h_T)]}{\Pr[\mathsf{EMCover}(S') = (h_1, \dots, h_T)]}
&=\frac{e^{\varepsilon q_t(S,h_t))}}{e^{\varepsilon q_t(S',h_t))}} \cdot \prod_{i = 1}^{t}\frac{\sum_{h \in H} e^{\varepsilon q_i(S', h)}}{\sum_{h \in H} e^{\varepsilon q_i(S, h)}}\\
&\le e^{\varepsilon} \cdot \prod_{i = 1}^t \frac{\sum_{h \in H_z} e^\varepsilon \cdot e^{\varepsilon q_i(S, h)} + \sum_{h \notin H_z} e^{\varepsilon q_i(S, h)}}{\sum_{h \in H} e^{\varepsilon q_i(S, h)}} \\ 
&= e^{\varepsilon} \cdot \prod_{i = 1}^t (1 + (e^\varepsilon - 1) p_i(S;z))\\
&\le e^{\varepsilon} \cdot \exp\left(\sum_{i = 1}^t2\varepsilon p_i(S;z)\right).
\end{align*}
The first equality uses the facts that i) after candidate $h_t$ is realized, the datasets consisting of uncovered elements in the two executions are identical, and ii) before the realization of $h_t$, by condition (2), scores for both datasets are the same. The first inequality is an immediate result of condition (1). Next, let $p_i(S;z)$ be the probability that the next candidate drawn covers $z$ when the dataset is $S$, conditioned on realizing $h_1, \dots, h_{i-1}$ in the previous iterations. Then the second inequality stems from the fact that for every $\varepsilon > 0$ and $x \in [0, 1]$, $1+(e^{\varepsilon}-1)x \le e^{2\varepsilon x}$. 

For some dataset $S$ and an example $z$, we call an output $\orgvec{h} = (h_1, h_2, ..., h_T)$ $\lambda$-bad if $\sum_{i = 1}^{T}p_i(S;z) \cdot \mathbbm{1}[\forall j \le i, h_j \text{ does not cover } z] > \lambda$. Otherwise, the output $\orgvec{h}$ is $\lambda$-good. We first consider the case when the output $\orgvec{h}$ is $\ln(\frac{1}{\delta})$-good. Then

\begin{align*}
\sum_{i=1}^{t-1}p_i(S;z)\le \sum_{i = 1}^{T}p_i(S;z) \cdot \mathbbm{1}[\forall j \le i, h_j \text{ does not cover } z] \le \ln(\frac{1}{\delta}).
\end{align*}

Thus, for every $\ln(\frac{1}{\delta})$-good output $h$, we have
\begin{align*}
\frac{\Pr[\mathsf{EMCover}(S) = (h_1, \dots, h_T)]}{\Pr[\mathsf{EMCover}(S') = (h_1, \dots, h_T)]} 
&\le e^{\varepsilon} \cdot e^{2\varepsilon \cdot (\ln(\frac{1}{\delta})+p_t(S;z))}\\
&\le e^{\varepsilon} \cdot e^{2\varepsilon \cdot (\ln(\frac{1}{\delta})+1)}\\
&\le e^{2\varepsilon \cdot (\ln(\frac{1}{\delta})+\frac{3}{2})}.
\end{align*}

Next, by invoking a martingale argument of~\citet{gupta2009differentially}, we show that the probability that $\mathsf{EMSetCover}$ outputs a $\ln(\frac{1}{\delta})$-bad output is bounded by at most $\delta$. 

\begin{lemma}~\citep{gupta2009differentially} \label{lemma:random-process}
Consider the $T$-round random process where an adversary chooses $p_i \in [0, 1]$ in each round possibly adaptively based on the previous $i-1$ rounds, and a coin is tossed with probability $p_i$ of realizing heads. Let $N_i$ be the indicator for the event that no heads are realized within the first $i$ steps. Defining the random variable $Y = \sum_{i=1}^T p_iN_i$,  we have $\Pr[Y > y] \leq \exp(-y)$ for every $y > 0$.
\end{lemma}

We map the setting of Lemma~\ref{lemma:random-process} to the execution of $\mathsf{EMCover}$ as follows. When selecting a candidate $h_i$ in round $i$, the algorithm first tosses a coin with heads probability of $p_i(S, z)$ to decide whether to pick a candidate that covers data point $z$ or not. Then, $\mathsf{EMCover}$ uses a second source of randomness to determine the output $h_i$, sampling it according to the right conditional probability based on the coin's outcome.

Therefore, for all sets of outputs $R$,
\begin{align*} 
\Pr[\mathsf{EMCover}(S) \in R] &= \sum_{\orgvec{h} \in R} \Pr[\mathsf{EMCover}(S) = \orgvec{h}] \\ 
&= \sum_{\substack{\orgvec{h} \in R, \\ \ln(\frac{1}{\delta})-\textrm{good for }S}} \Pr[\mathsf{EMCover}(S) = \orgvec{h}] + \sum_{\substack{\orgvec{h} \in R, \\ \ln(\frac{1}{\delta})-\textrm{bad for }S}} \Pr[\mathsf{EMCover}(S) = \orgvec{h}] \\ 
&\le \sum_{\orgvec{h} \in R, \ \ln(\frac{1}{\delta})-\textrm{good for }S} e^{2\varepsilon \cdot (\ln(\frac{1}{\delta})+\frac{3}{2})} \Pr[\mathsf{EMCover}(S') = \orgvec{h}] + \delta \\ 
&\le e^{2\varepsilon \cdot (\ln(\frac{1}{\delta})+\frac{3}{2})} \Pr[\mathsf{EMCover}(S') \in R] + \delta.
\end{align*}

Studying the opposite ratio proceeds similarly.

Case (b)
\begin{align*}
\frac{\Pr[\mathsf{EMCover}(S') = (h_1, \dots, h_T)]}{\Pr[\mathsf{EMCover}(S) = (h_1, \dots, h_T)]}
&=\frac{e^{\varepsilon q_t(S',h_t))}}{e^{\varepsilon q_t(S,h_t))}} \cdot \prod_{i = 1}^{t}\frac{\sum_{h \in H} e^{\varepsilon q_i(S, h)}}{\sum_{h \in H} e^{\varepsilon q_i(S', h)}}\\
&\le e^{\varepsilon} \cdot \prod_{i = 1}^t \frac{\sum_{h \in H_z} e^\varepsilon \cdot e^{\varepsilon q_i(S', h)} + \sum_{h \notin H_z} e^{\varepsilon q_i(S', h)}}{\sum_{h \in H} e^{\varepsilon q_i(S', h)}} \\ 
&= e^{\varepsilon} \cdot \prod_{i = 1}^t (1 + (e^\varepsilon - 1) p_i(S';z))\\
&\le e^{\varepsilon} \cdot e^{\sum_{i = 1}^t2\varepsilon p_i(S';z)}\\
\end{align*}
Similar to Case (a), the first equality uses the facts that i) after candidate $h_t$ is realized, the datasets of uncovered elements in the two executions are identical, and ii) before the realization of $h_t$, by condition (2), scores for both datasets are the same. The first inequality is derived from the double inequality for quality scores from (1). Then, we denote the probability that the next candidate drawn covers $z$ when the dataset is $S'$ conditioned on realizing $h_1, \dots, h_{i-1}$ in the previous iterations as $p_i(S';z)$. Again using the fact that for every $\varepsilon > 0$ and $x \in [0, 1]$, $1+(e^{\varepsilon}-1)x \le e^{2\varepsilon x}$, we obtain the second inequality.

Next, we perform an analogous case analysis on whether the output $\orgvec{h}$ is $\ln(\frac{1}{\delta})$-good and bad.

If $\orgvec{h} = (h_1, \dots, h_T)$ is  $\ln(\frac{1}{\delta})$-good, we have 
\begin{align*}
\frac{\Pr[\mathsf{EMCover}(S') = (h_1, \dots, h_T)]}{\Pr[\mathsf{EMCover}(S) = (h_1, \dots, h_T)]} 
&\le e^{\varepsilon} \cdot e^{2\varepsilon \cdot (\ln(\frac{1}{\delta})+p_t(S';z))}\\
&\le e^{\varepsilon} \cdot e^{2\varepsilon \cdot (\ln(\frac{1}{\delta})+1])}\\
&\le e^{2\varepsilon \cdot (\ln(\frac{1}{\delta})+\frac{3}{2})}
\end{align*}
Otherwise, using Lemma~\ref{lemma:random-process}, we have $ \sum_{\substack{\orgvec{h} \in R, \ \ln(\frac{1}{\delta})-\textrm{bad for } S'}}\Pr[\mathsf{EMCover}(S') = \orgvec{h}] \le \delta$. Thus, for every set of outcomes $R$,
\begin{align*} 
\sum_{\orgvec{h} \in R} \Pr[\mathsf{EMCover}(S') = \orgvec{h}] &= \sum_{\substack{\orgvec{h} \in R, \ \ln(\frac{1}{\delta})-\textrm{good for }S'}} \Pr[\mathsf{EMCover}(S') = \orgvec{h}] \\
&\phantom{{}={}} + \sum_{\substack{\orgvec{h} \in R, \ \ln(\frac{1}{\delta})-\textrm{bad for }S'}} \Pr[\mathsf{EMCover}(S') = \orgvec{h}] \\ 
&\phantom{{}\le{}} \sum_{\substack{\orgvec{h} \in R, \ \ln(\frac{1}{\delta})-\textrm{good for }S'}} e^{2\varepsilon \cdot (\ln(\frac{1}{\delta})+\frac{3}{2})} \Pr[\mathsf{EMCover}(S) = \orgvec{h}] + \delta \\ 
&\le e^{2\varepsilon \cdot (\ln(\frac{1}{\delta})+\frac{3}{2})} \Pr[\mathsf{EMCover}(S) \in R] + \delta.
\end{align*}
\end{proof}

\subsection{Utility Analysis}

Lemma~\ref{lemma: DP-GreedyCover Empirical err} describes the empirical error bound for $\mathsf{DP\text{-}GreedyCover}$.

\begin{lemma}\label{lemma: DP-GreedyCover Empirical err}
     Fix a target function $c^*\in \mathcal{F}$-decision lists with feature set size of $M$, and consider the execution of $\mathsf{DP\text{-}GreedyCover}$ on a sample $S = \{ (x_i, \sigma_i)\}_{i = 1}^n$. Assume $\beta >0$, and for every iteration $j$. Then, with probability at least $1-\beta$, the algorithm errs on at most $\frac{4M}{\varepsilon}(\log({\sqrt{\frac{2}{\beta}}M}))$ examples in $S$. 
\end{lemma}


\begin{proof}
By Proposition~\ref{proposition:consistent k-decision list}, every iteration maintains the consistency between the remaining sample $S$ and some decision list with a feature set of size $M-1$. This implies every iteration $j$ there is a pair of Boolean function $f \in \mathcal{F}_{j}$ and Boolean $b$ such that either $\#_{f \rightarrow 1}(S^1) = 0$ or $\#_{f \rightarrow 1}(S^0) = 0$. By Proposition~\ref{proposition: Exponential Mechanism Utility}, for every iteration $j$, the mechanism outputs a pair $(f_j, b_j)$ that satisfies for all $t\geq0$,

\[\Pr[q(S_j, (f_j, b_j)) \ge \underset{(f, b) \in \mathcal{F}_{j} \times \{0,1\}}{\max}\{q(S_j,(f, b))\} - \frac{2}{\varepsilon}(\log(|\mathcal{F}_{j}\times\{0,1\}|)+ t)] \leq e^{-t}.\]

Let $t = \log(\frac{M}{\beta})$. Then, by union-bounding over the $M$ iterations and taking the complement, we have 

\begin{align*}
\Pr\Biggl[\, & \bigcap_{j=1}^{M} \Biggl(
      q(S_j, (f_j,b_j)) \ge 
      \max_{(f,b)\in \mathcal{F}_j\times\{0,1\}} q(S_j,(f,b)) -\,\frac{2}{\varepsilon}\Biggl(
            \log\Bigl(|\mathcal{F}_j\times\{0,1\}|\Bigr)
            + \log\Bigl(\frac{M}{\beta}\Bigr)
      \Biggr)
\Biggr) \,\Biggr] \\
&\ge 1 - M\,e^{-\log\Bigl(\frac{M}{\beta}\Bigr)} \\
&= 1-\beta.
\end{align*}

Furthermore, since there exists a pair $(f,b) \in \mathcal{F}_{j}\times\{0,1\}$ such that either $\#_{f \rightarrow 1}(S^1) = 0$ or $\#_{f \rightarrow 1}(S^0)  = 0$, we have



\[
\Pr[\bigcap_{j = 1}^{M} -q(S_j, (f_j,b_j)) \leq \frac{2}{\varepsilon}(\log(|\mathcal{F}_{j} \times\{0,1\}|)+ \log(\frac{M}{\beta}))] \geq 1 - \beta
\]
\[
\implies \Pr[\bigcap_{j = 1}^{M} \underset{f_j \rightarrow 1}{\#}(S_{j}^{1-b_j}) \leq \frac{4}{\varepsilon}(\log(\sqrt{\frac{2}{\beta}}M))] \geq 1-\beta
\]

In every iteration $j$, the selected pair of Boolean function and Boolean $(f_j,b_j)$ misclassifies at most $\frac{4}{\varepsilon}(\log({\sqrt{\frac{2}{\beta}}M}))$ examples from $S_j$ with probability at least $1-\beta$. Furthermore, since the algorithm exhausts all candidate functions by the end of iteration $M$, all examples from $S$ have been classified. Therefore, $h_{fin}$ errs on at most $\frac{4M}{\varepsilon}(\log({\sqrt{\frac{2}{\beta}}M}))$ examples of $S$ with probability at least $1-\beta$.
\end{proof}

Having established that our learner enjoys low empirical error with respect to its sample, we next quote the VC dimension generalization bound we need to translate this into low population error.

\begin{theorem}(VC-Dimension Generalization Bound\label{theorem: VC Dimension Generalization Bound}~\citep{BlumerEHW89})
Let $\mathcal{D}$ and $\mathcal{C}$ be a distribution and concept class over domain $\mathcal{X}$ respectively, and let $c \in \mathcal{C}$. For sample $S = \{(x_i, c(x_i))\}_{i = 1}^{n}$ of size $ n \geq \frac{64}{\alpha}(VC(\mathcal{C})\log(\frac{64}{\alpha}) + \log(\frac{16}{\beta}))$ where $x_i$ are chosen i.i.d from $\mathcal{D}$, it holds that
$$\Pr\limits_{S \sim \mathcal{D}}[\exists h \in \mathcal{C} \textrm{ s.t. } err_\mathcal{D}(h, c) > \alpha, err_S(h) \le \frac{\alpha}{2} ] \leq \beta$$
\end{theorem}

\begin{proof}[Proof of Theorem~\ref{theorem: zCDP DP-GreedyCover}]\label{proof: DP-GreedyCover Privacy and Sample Complexity}

Let algorithm $A$ denote $\mathsf{DP\text{-}GreedyCover}$ with the Exponential Mechanism instantiated with $\hat{\varepsilon} = \frac{\varepsilon}{2(\log(\frac{1}{\delta}) + \frac{3}{2})}$. By Claim~\ref{lemma:Quality constraint implies DP}, $A$ is $(\varepsilon, \delta)$-DP.

Assume that the number of samples $n$ is at least $\frac{64}{\alpha}(VC(\mathcal{C})\log(\frac{64}{\alpha}) + \log(\frac{16}{\beta}))$. By Lemma~\ref{lemma: DP-GreedyCover Empirical err}
$\mathsf{DP\text{-}GreedyCover}$ obtains generalization error at most $\frac{8M\log(\frac{2M}{\sqrt{\beta}})}{\varepsilon n}$ with probability at least $1-\beta$. Then, we compute the sample complexity bound

\begin{align*}
n \ge \frac{8M\log(\frac{2M}{\sqrt{\beta}})}{\hat{\varepsilon}\alpha}= \frac{8M\log(\frac{2M}{\sqrt{\beta}})(2\log(\frac{1}{\delta})+\frac{3}{2})}{\alpha \varepsilon}.
\end{align*}
\end{proof}

\subsection{Application of $\mathsf{DP\text{-}GreedyCover}$ to Learning $k$-Decision Lists}

In this section, we illustrate using $\mathsf{DP\text{-}GreedyCover}$ to learn the class of $k$-decision lists, where each term is a conjunction of at most $k$ literals over $d$ Boolean variables. We denote the set of conjunctions of at most $k$ literals over $d$ Boolean variables as $\mathcal{C}_{d}^k$. 

\begin{corollary}\label{corollary: DP-GreedyCover k-decision list}
There exists a $(\varepsilon, \delta)$-DP $(\alpha,\beta)$ PAC learner for the concept class $k$-decision lists with sample complexity 
$$ n = O\left( \frac{d^k \log(d^k/\beta)\cdot \log(\frac{1}{\delta})}{\alpha \varepsilon}\right)$$
\end{corollary}

To see this, it suffices to show that the number of conjunctions with at exactly $k$ literals over $d$ Boolean variables is upper bounded by $e^{2}d^k$ in Proposition~\ref{proposition:Cardinality of Set of Sparse Conjunctions}.

\begin{proposition}\label{proposition:Cardinality of Set of Sparse Conjunctions}
$|\mathcal{C}_{d}^k| \leq e^2d^k$. 
\end{proposition}
\begin{proof}
\begin{align*}
|\mathcal{C}_{d}^k| &= \sum_{j = 0}^k \#\mathrm{\ conjunctions\ of\ width\ exactly } j \\
&= \sum_{j = 0}^k \#\mathrm{\ ways\ to\ choose\ } j \mathrm{\  variables } \times \#\mathrm{\ ways\ to\ choose\ sign\ pattern} \\
&= \sum_{j = 0}^k 2^j \binom{d}{j} \\
&\le \sum_{j = 0}^k \frac{2^j}{j!} (d) \cdot (d-1) \cdot ... \cdot( d-k+1) \\
&\le d^k \sum_{j = 0}^k \frac{2^j}{j!} \le d^k \sum_{j = 0}^{\infty} \frac{2^j}{j!} \le e^2d^k.
\end{align*}
\end{proof}

\section{Proof of Lemma~\ref{lemma:ConfidentWinnowUpdatesUpperBound}}
\label{appendix:confident-winnow}
\begin{proof}
For every $t \in [T]$, define the potential function $\phi^{(t)}$ to be the relative entropy between the components of the normalized target vector $v$ and the components of the hypothesis vector $w^{(t)}$: 
\[\phi^{(t)} = \sum_{i = 1}^d v_i \log\left(\frac{v_i}{w^{(t)}_i}\right).\]
Intuitively, this captures how close the hypothesis weight vector $w^{(t)}$ is to the target weight vector $v$.

Suppose the algorithm performs an update in round $t+1$. Let us compute the difference between $\phi^{(t+1)}$ and $\phi^{(t)}$, i.e, how much the potential decreases as a result of this update.

\begin{align*}
     \phi^{(t+1)} - \phi^{(t)} & = \sum_{i = 1}^d v_i \left(\log\left(\frac{v_i}{w^{(t+1)}_i}\right) - \log\left(\frac{v_i}{w^{(t)}_i}\right)\right) \\
    & = \sum_{i = 1}^d v_i \log\left(\frac{w^{(t)}_i}{w^{(t+1)}_i}\right) \\
    & = \sum_{i = 1}^d v_i \log\left(\frac{\sum_{k=1}^{d}w^{(t)}_k e^{\eta y^{(t)} x^{(t)}_k}}{e^{\eta y^{(t)} x^{(t)}_i}}\right) \\
    & = \sum_{i = 1}^d v_i\log\left(\sum_{k =1}^{d}w^{(t)}_k e^{\eta y^{(t)} x^{(t)}_k}\right) +
    \sum_{i = 1}^d v_i \log\left(\frac{1}{e^{\eta y^{(t)} x^{(t)}_i}}\right) \\
    & = \log\left(\sum_{k=1}^{d}w^{(t)}_k e^{\eta y^{(t)} x^{(t)}_k}\right) -
    \sum_{i = 1}^d v_i \eta y^{(t)} x^{(t)}_i \\
    & \le \log\left(\sum_{k =1}^{d}w^{(t)}_k e^{\eta y^{(t)} x^{(t)}_k}\right) -
    \eta \rho \\
    & = \log\left(\mathbb{E}_{k\sim w^{(t)}}[e^{\eta y^{(t)} x^{(t)}_k}]\right) -
    \eta \rho \\
    & = \log\left(\mathbb{E}_{k\sim w^{(t)}}[\exp(\eta y^{(t)} x^{(t)}_k - \eta y^{(t)} \langle w^{(t)}, x^{(t)} \rangle + \eta y^{(t)} \langle w^{(t)}, x^{(t)} \rangle)]\right) -
    \eta \rho \\
    & \le \log\left(\mathbb{E}_{k\sim w^{(t)}}[\exp(\eta^2 / 2)]\right) + \eta y^{(t)} \langle w^{(t)}, x^{(t)} \rangle -
    \eta \rho \\
    & \le \frac{\eta^2}{2} + \eta\rho(c-1).
\end{align*}

The first inequality uses the margin assumption that $\sum_{i =1}^{d} y^{(t)}v_ix^{(t)}_i \ge \rho$. The following two inequalities are derived by expressing the sum as an expectation over the distribution defined by $w^{(t)}$. The second inequality follows by applying Hoeffding's lemma to the random variable $x^{(t)}_{k} \in \{-1, 1\}$ where $k \sim w^{(t)}$:
\begin{align*}
    \mathbb{E}_{k\sim w^{(t)}}[\exp(\eta y^{(t)}(x^{(t)}_k - \langle w^{(t)}, x^{(t)}\rangle))]
    &= \mathbb{E}_{k\sim w^{(t)}}[\exp(\eta y^{(t)}(x_k^{(t)} -\mathbb{E}[x_k^{(t)}])]\\
    &= \exp(\frac{\eta^2 \cdot (1 - (-1))^2}{8}) = \exp(\frac{\eta^2}{2}).
\end{align*}
Finally, the last inequality uses the fact that updates may only occur when the prediction is unconfident or incorrect, i.e., when $y^{(t)}\langle w^{(t)}, x^{(t)}\rangle \le c\rho$. 

Let $M$ be the total number of updates the algorithm performs. Then summing up over all of the rounds in which an update occurs, we have 
    $$\phi_{T+1} - \phi_{1}\le M(\frac{\eta^2}{2} + \eta\rho(c-1)).$$
    
To derive an upper bound on $M$, observe that
\begin{enumerate}
    \item $\phi_{T+1} \ge 0$ since the relative entropy between two distributions is always nonnegative, and
    \item $\phi_{1} =  \sum_{i = 1}^d v_i\log(\frac{v_i}{1/d}) \le \log d$.
\end{enumerate}

By combining these inequalities, we obtain the stated upper bound on $M$.
\begin{align*}
&-\log d \le M(\frac{\eta^2}{2} + \eta\rho(c-1))\\
\implies&
 M \le \frac{\log d}{(1-c)\eta\rho - \frac{\eta^2}{2}}.
\end{align*}
\end{proof}

\section{Proof of Theorem~\ref{thm:winnow-privacy}} \label{appendix:online-privacy}

Our $\mathsf{DP\text{-}Winnow}$ algorithm is perhaps most naturally viewed as a concurrent composition between two interactive differentially private algorithms, described at a high-level as follows.

\paragraph{View 1: $\mathsf{DP\text{-}Winnow}$ as a concurrent composition.}{
Run the following two algorithms concurrently:
\begin{enumerate}
    \item Sequentially compose at most $K$ invocations of the $\mathsf{AboveThreshold}$ algorithm to reveal the rounds in which an update to the weight vector should occur.
    \item Maintain the weight vector $w$, and whenever $\mathsf{AboveThreshold}$ triggers an update, perform it and release $m$ independent samples from the Exponential Mechanism.
\end{enumerate}
}

This view of the algorithm allows one to use concurrent composition theorems for differential privacy~\citep{VadhanW21, VadhanZ23} to reason about its privacy guarantees. However, it is possible to give a simpler analysis by studying the following equivalent algorithm that just post-processes a sequential composition of $\mathsf{AboveThreshold}$ instances with the Exponential Mechanism.

\paragraph{View 2: Recompute weights from history}{
Consider the following algorithm whose output distribution is identically distributed to that of Algorithm~\ref{alg: DP-WinnowHalfSpace}, but does so with the following changes:
\begin{enumerate}
    \item At the end of each round $t$, it discards its current weights $w^{(t)}$.
    \item If in a round $t$, the above-threshold check passes in line 15, it recomputes $w^{(t)}$ from scratch by rerunning the same multiplicative weight updates on the prefix $(x^{(1)}, y^{(1)}), \dots, (x^{(t)}, y^{(t)})$ that $\mathsf{DP\text{-}Winnow}$ would have performed given the history of released approximate weight vectors $\tilde{w}^{(1)}, \dots, 
    \tilde{w}^{(t-1)}$.
    \item It releases $\tilde{w}^{(t)}$ either by sampling from $w^{(t)}$ (i.e., using the Exponential Mechanism) or by keeping $\tilde{w}^{(t-1)}$, as appropriate.
\end{enumerate}
}

\begin{proof}

Since this alternative implementation of the algorithm does not maintain private state between $\mathsf{AboveThreshold}$ and Exponential Mechanism invocations, it can be viewed as the sequential composition of at most $K$ applications of $\mathsf{AboveThreshold}$,  each $(\hat{\varepsilon}, 0)$-DP, and $mK$ Exponential Mechanism invocations, each $(2\eta)$-DP.

To see the latter, let us focus on the computation performed by $\mathsf{DP\text{-}Winnow}$ between lines \ref{alg:DP-Winnow-weights-update-start} and \ref{alg:DP-Winnow-weights-update-end}, and fix a history of the algorithm's outputs $\tilde{w}^{(1)}, \dots, \tilde{w}^{(t)}$ and updates. We claim that each sample $j_i$ in line~\ref{alg:DP-Winnow-weights-sample} corresponds to an invocation of the Exponential Mechanism on a sensitivity-1 score function with parameter $\eta$. To see this, observe that we can unroll the numerator of each entry $w_j$ of the weight vector as $\exp(\sum_{i \in R} \eta \hat{y}^{(i)} \hat{x}^{(i)}_j)$ where $R$ is the set of indices of examples which were used to perform updates. For any fixed history of weights and updates, and for any pair of neighboring example sequences, there is at most one index in $R$ and at most one corresponding term in the sum that differ. Thus, drawing a single sample $j_i$ is $(2\eta)$-differentially private. 

Let $\varepsilon_{AT}$ and  $\varepsilon_{EM}$ be the privacy loss contributed by $\mathsf{AboveThreshold}$ and Exponential Mechanism calls. Applying the advanced composition theorem~\citep{DworkR14} with failure probability split evenly between the two layers ($\delta/2$ each), one obtains
\begin{enumerate}
\item $\varepsilon_{AT} \le \hat{\varepsilon}\sqrt{2K\log(2/\delta)} + K\hat{\varepsilon}^2$, and
\item $\varepsilon_{EM} \le 2\eta\sqrt{2mK\log(2/\delta)} + 4mK\eta^2$.
\end{enumerate}
Setting 
\[\hat{\varepsilon} = \varepsilon/(4\sqrt{2K \log (2/\delta)}), \qquad \eta = \varepsilon/(8\sqrt{2mK\log(2/\delta)})
\]
yields


\[\varepsilon_{AT} + \varepsilon_{EM} = \hat{\varepsilon}\sqrt{2K\log(2/\delta)} + 2\eta\sqrt{2mK\log(2/\delta)} + K\hat{\varepsilon}^2 + 4mK\eta^2\]
\[\le \varepsilon/4 + \varepsilon/4 + \varepsilon^2/(16\log(2/\delta)) < \varepsilon/2+ \varepsilon/8
< \varepsilon\]
 for all $\varepsilon, \delta \in (0,1)$ and gives the stated bound. 
\end{proof}

\section{Proof of Lemma~\ref{lemma: high probability argument matching shadow and approximate weight vectors}}\label{appendix:winnow-match}
\begin{proof}
Fix a time step $t \in [T]$. Our goal is to show that if $w$ is any shadow weight vector with $\|w\|_1 = 1$ and $x \in \{-1, 1\}^d$, then
\begin{align*}
    \langle w, x\rangle \ge c\rho &\implies \langle \tilde{w}, x\rangle > 0 \\
    \langle w, x\rangle \le -c\rho &\implies \langle \tilde{w}, x\rangle < 0
\end{align*}
with probability at least $1 - \beta/T$ over the sampling of $\tilde{w}$.

For each $i \in [m]$, let $J_i$ be the $i$'th index sampled. That is, for each index $j \in [d]$, we have $\Pr[J_i = j] = w_j$. For each $i \in [m]$, define the random variable $\chi_i = x_{J_i}$, which is $\pm 1$-valued with $\mathbb{E}[\chi_i] = \sum_{j = 1}^d x_j \Pr[J_i = j] = \langle w, x \rangle$.


For each $j \in [d]$, let $\tilde{w}_j = \frac{1}{m} \sum_{i = 1}^m \mathbf{1}[J_i = j]$ be the fraction of times $j$ is sampled, and let $\tilde{w} = (\tilde{w}_1, \dots, \tilde{w}_d)$.  Then

\begin{align*}
    \langle \tilde{w}, x\rangle &= \frac{1}{m} \sum_{j = 1}^d x_j \sum_{i=1}^m \mathbf{1}[J_i = j] \\
    &= \frac{1}{m} \sum_{i = 1}^m x_{J_i} \\
    &= \frac{1}{m} \sum_{i = 1}^m \chi_{i}.
\end{align*}

Thus, we can estimate the probability of the bad event where $\langle \tilde{w}, x \rangle$ disagrees in sign with $\langle w, x \rangle$ as
\begin{align*}
    \Pr[|\langle \tilde{w},x\rangle - \langle w,x\rangle| > |\langle w,x \rangle|] &\le     \Pr[|\langle \tilde{w},x\rangle - \langle w,x\rangle| > c \rho]\\
    &=\Pr\left[\left|\frac{1}{m}\left(\sum_{i =1}^{m}\chi_i - \sum_{i=1}^{m}\mathbb{E}[\chi_i]\right)\right| > c\rho \right]\\
    &= \Pr\left[\left|\sum_{i =1}^{m}\chi_i - \sum_{i=1}^{m}\mathbb{E}[\chi_i]\right| >mc\rho \right]\\ 
    &\le 2\exp\left(-\frac{mc^2\rho^2}{2}\right)\\ 
    &\le \frac{\beta}{T}
\end{align*}

The first inequality uses our confidence assumption that $|\langle w,x \rangle| \ge c\rho$. The second inequality follows from Hoeffding's inequality. We derive the last inequality by setting $m \ge \frac{2}{c^2\rho^2}\log(\frac{2T}{\beta})$.
\end{proof}

\end{document}